\documentclass[reqno]{amsart}
\usepackage{xcolor}
\usepackage{graphicx}
\theoremstyle{plain}
\begingroup
 
\newtheorem{theorem}{Theorem} 
\newtheorem{corollary}{Corollary}
\newtheorem{lemma}{Lemma}
\endgroup
\newcommand{\nwc}{\newcommand}
\nwc{\qref}[1]{(\ref{#1})}
\nwc{\cadlag}{c\`{a}dl\`{a}g}
\nwc{\la}{\label}
\nwc{\nn}{\nonumber}
\nwc{\Z}{\mathbb{Z}}
\nwc{\C}{\mathbb{C}}
\nwc{\T}{\mathbb{T}}
\nwc{\E}{\mathbb{E}}
\nwc{\R}{\mathbb{R}}
\nwc{\N}{\mathbb{N}}
\nwc{\Rn}{\mathbb{R}^n}
\nwc{\PP}{\mathcal{P}}
\nwc{\M}{\mathcal{M}}
\nwc{\Ito}{It\^{o}}
\nwc{\DiffM}{\mathrm{Diff}(M)}
\nwc{\orbit}{\mathcal{O}}
\nwc{\bbb}{\mathbf{M}}

\nwc{\law}{\stackrel{\mathcal{L}}{\rightarrow}}
\nwc{\eqd}{\stackrel{\mathcal{L}}{=}}
\nwc{\vp}{\varphi}
\nwc{\veps}{\varepsilon}
\nwc{\eps}{\veps}
\nwc{\dnto}{\downarrow}
\nwc{\nsup}{^{(n)}}
\nwc{\ksup}{^{(k)}}
\nwc{\jsup}{^{(j)}}
\nwc{\nksup}{^{(n_k)}}
\nwc{\inv}{^{-1}}
\nwc{\argmin}{\mathrm{argmin}}
\nwc{\argmax}{\mathrm{argmax}}
\nwc{\Tr}{\mathrm{Tr}}
\nwc{\Id}{\mathrm{Id}}
\nwc{\Pn}{\mathbb{P}(n)}
\nwc{\PnR}{\mathbb{P}(n;\mathbb{R})}
\nwc{\PnC}{\mathbb{P}(n;\mathbb{C})}
\nwc{\Hn}{\mathbb{H}(n)}
\nwc{\HnR}{\mathbb{H}(n;\mathbb{R})}
\nwc{\HnC}{\mathbb{H}(n;\mathbb{C})}
\nwc{\An}{\mathbb{A}(n)}
\nwc{\AnR}{\mathbb{A}(n;\mathbb{R})}
\nwc{\AnC}{\mathbb{A}(n;\mathbb{C})}
\nwc{\Mn}{\mathbb{M}(n)}
\nwc{\Md}{\mathbb{M}_d}
\nwc{\Mfd}{\mathfrak{M}_d}
\nwc{\Mm}{\mathbb{M}_m}
\nwc{\Mfm}{\mathfrak{M}_m}
\nwc{\feasible}{\mathcal{F}}
\nwc{\GLn}{GL(n)}
\nwc{\GLnC}{GL(n;\mathbb{C})}
\nwc{\GLnR}{GL(n;\mathbb{R})}
\nwc{\Poincare}{Poincar\'{e}}
\nwc{\Un}{U(n)}
\nwc{\Sn}{\mathbb{S}^{n}}
\nwc{\Tn}{\mathbb{T}^n}
\nwc{\ev}{\mathrm{ev}}
\nwc{\Diff}{\mathrm{Diff}}
\nwc{\orbitx}{\mathcal{O}_X}
\nwc{\orbitxG}{\mathcal{O}_{X,\mathbf{G}}}
\nwc{\volume}{\mathrm{vol}}
\nwc{\symm}{\mathrm{Symm}}
\nwc{\psd}{\mathbb{P}}
\nwc{\od}{O_d}
\nwc{\balance}{\mathcal{M}}
\nwc{\balanceG}{\mathcal{M}_{\mathbf{G}}}
\nwc{\nstar}{n_*}
\nwc{\manmap}{\mathfrak{x}}
\nwc{\manmapbeta}{\mathfrak{y}}
\nwc{\grad}{\mathrm{grad}}
\nwc{\ww}{\mathbf{W}}
\nwc{\ff}{\mathcal{F}}
\nwc{\ggg}{\mathcal{G}}
\nwc{\hh}{\mathcal{H}}
\nwc{\ts}{n}
\nwc{\anti}{\mathbb{A}}
\nwc{\diag}{\mathrm{diag}}
\nwc{\Lojas}{Lojasiewicz}
\nwc{\bd}{\mathbf}

\nwc{\Mr}{\mathfrak{M}_r}
\nwc{\Mdd}{\mathfrak{M}_d}
\nwc{\bfW}{\mathbf{W}}
\nwc{\bfw}{\mathbf{w}}
\nwc{\bfA}{\mathbf{A}}
\nwc{\bfa}{\mathbf{a}}
\nwc{\bfaa}{\mathbf{alpha}}
\nwc{\bfQ}{\mathbf{Q}}
\nwc{\bfC}{\mathbf{C}}
\nwc{\bfc}{\mathbf{c}}
\nwc{\bfl}{\mathbf{l}}
\nwc{\bfm}{\mathbf{m}}
\nwc{\bfu}{\mathbf{u}}
\nwc{\bfv}{\mathbf{v}}
\nwc{\gbw}{g^{BW}}
\nwc{\fiber}{\mathcal{F}}
\nwc{\fiberX}{\mathcal{F}_X}
\nwc{\aaa}{\mathcal{A}}
\nwc{\Stief}{\mathrm{St}}
\nwc{\Fr}{\iota}
\nwc{\weyl}{\mathcal{W}}
\nwc{\Her}{\mathrm{Her}}
\nwc{\refspace}{\mathcal{M}}
\nwc{\refgroup}{\mathcal{G}}
\nwc{\quotientspace}{\mathcal{X}}
\nwc{\refmetric}{g}
\nwc{\quotientmetric}{h}

\theoremstyle{definition}
 
\newtheorem{remark}[theorem]{Remark}

\theoremstyle{remark}
 
\numberwithin{equation}{section}
\numberwithin{figure}{section}

\begin{document}


\title{An entropy formula for the Deep Linear Network}

\begin{abstract}
We study the Riemannian geometry of the Deep Linear Network (DLN) as a foundation for a thermodynamic description of the learning process. The main tools are the use of group actions to analyze overparametrization and the use of Riemannian submersion from the space of parameters to the space of observables. The foliation of the balanced manifold in the parameter space by group orbits is used to define and compute a Boltzmann entropy. We also show that the Riemannian geometry on the space of observables defined in~\cite{Bah} is obtained by Riemannian submersion of the balanced manifold. The main technical step is an explicit construction of an orthonormal basis for the tangent space of the balanced manifold using the theory of Jacobi matrices.
\end{abstract}

\author{Govind Menon}
\address{Division of Applied Mathematics, Brown University, 182 George St., Providence, RI 02912.}
\email{govind\_menon@brown.edu}
\curraddr{School of Mathematics, Institute for Advanced Study, 1 Einstein Dr., Princeton, NJ 08540.}
\email{gmenon@ias.edu}

\author{Tianmin Yu}
\address{Department of Mathematics, Northwestern University, 2033 Sheridan Rd., Evanston, IL 60208.}
\email{tianmin.yu@northwestern.edu}

\thanks{This work was supported by the NSF grants DMS 2107205, DMS 2407055 and the Erik Ellentuck Fellowship at the Institute for Advanced Study, Princeton. GM is also grateful to the Institute of Mathematics, Academia Sinica for partial support under grant NSTC 113-2115-M-001-009-MY3 during the completion of this work.}
\keywords{Deep linear network, Random matrix theory, Kalman's realizability theory}

\maketitle

{\centering \em In memory of Roger Brockett. \par}

\section{Overview}
The Deep Linear Network (DLN) is a phenomenological model of deep learning based on the composition of {\em linear\/} functions. It was introduced by Arora, Cohen and Hazan in 2017~\cite{ACH}. The DLN builds on earlier models in neural computing of a similar nature (in particular~\cite{Baldi-Hornik}), but differs in its emphasis on the role of depth. 


In this paper, and the companion papers~\cite{LM1,MY-RLE}, we use the DLN to develop a rigorous thermodynamic description of the learning process. The main novelty in our work is the systematic use of the geometric theory of dynamical systems. This is in contrast with the more common approach based on statistical learning theory. We study both the Riemannian and symplectic geometry that underlie training dynamics and use it as the foundation for a thermodynamic description. 

This paper develops the Riemannian geometry. We provide explicit descriptions of the underlying Riemannian metrics and define a natural Boltzmann entropy. The related stochastic dynamics are described by a Riemannian Langevin equation (RLE) which is presented in the companion paper~\cite{MY-RLE}. Both these papers rely on analogies with random matrix theory. The determinantal formula for the entropy in this paper (Theorem~\ref{thm:entropy}) is a variant of well-known formulas in random matrix theory. The RLE for the DLN in~\cite{MY-RLE} is a geometric stochastic flow analogous to Dyson Brownian motion. In the companion paper~\cite{LM1}, we develop the symplectic geometry of the DLN and provide a variational characterization of balancedness. This result is obtained by developing an analogy with Kalman's minimal realizability theory in linear systems theory. Both these analogies reveal unexpected mathematical structure in the DLN and new directions for study.

Our broader goal in this work is to obtain insight into training by stochastic gradient descent in deep learning using geometric methods. While our work does not yet include stochastic gradient descent, it is a step in this direction. Redundancy of description (there are many choices of parameter that create the same function) is a central feature of deep learning. Our analysis of the DLN reveals the manner in which symmetry may be used to analyze such overparametrization, extending the classical theory of gradient flows to the high-dimensional setting of deep learning. 

The use of group actions and Riemannanian geometry builds on earlier attempts in control and systems theory, especially by Bloch, Brockett and their co-authors, to develop gradient flows that solve combinatorial optimization problems and logical tasks~\cite{Bloch1994,BBR,Brockett-automata}. Brockett's work was also motivated by the study of neural networks and blurs the usual distinction between analog and digital computing~\cite{Brockett-transient,Brockett}. The reappearance of Kalman's realizability theory in the DLN reveals the continued utility of geometric control theory in the analysis of neural networks.  

The Riemannian geometry studied here is of independent mathematical interest. An explicit description of geodesics in the DLN geometry based on the classical theory of confocal quadrics and Theorem~\ref{thm:submersion} below is presented by Chen~\cite{Chen1}. His work generalizes Brenier's formula for geodesics in the Bures-Wasserstein metric to the DLN providing a new connection between optimal transport theory and deep learning.  The entropy formula in this paper has also been used to study equilibrium measures in the DLN~\cite{Chen2}, formalizing the notion of entropic regularization and selection of minimizers by small noise in deep learning. A surprising feature of the DLN, in contrast with random matrix theory, is that the equilibrium measures do {\em not\/} show repulsion between singular values. Finally, the use of Riemannian gradient flows provides a unity between interior point methods for convex optimization and deep learning as discussed in~\cite{MY-conic}.\

This paper is organized as follows. We review some prior work in Section~\ref{sec:background} and state our main results in Section~\ref{sec:results}. Group actions and the foliation of the balanced manifold by group orbits is discussed in Section~\ref{sec:group-action}. The entropy formula is proved in Section~\ref{sec:orbit-basis} and Riemannian submersion is established in Section~\ref{sec:balance-basis}. In both these sections, we find a subtle coupling in depth along the network that is reflected in the description of the Riemannian metric on the balanced manifold by certain block Jacobi matrices. The explicit diagonalization of these matrices yields the formulas in this paper. The paper concludes with a discussion (Section~\ref{sec:discussion}) that includes references to related work and a summary of the main new insights.

\section{Background}
\label{sec:background}
We begin by reviewing rigorous results by other authors that have stimulated our work. The expository article~\cite{GM-dln} provides more details on sections~\ref{subsec:model}--\ref{subsec:prior} below. 

\subsection{Notation} 
\label{subsec:notation}
We fix two positive integers $N$ and $d$ referred to as the depth and width of the network respectively and an integer $r=1,\cdots, d$. We denote by $\Md$, $\Mfd$, $\Mr$, $\symm_d$, $\anti_d$ and $\psd_d$ the spaces of $d\times d$ real matrices, invertible matrices in $\Md$, rank-$r$ matrices in $\Md$, symmetric matrices, anti-symmetric matrices, and positive semidefinite (psd) matrices respectively. We denote by $O_d$ the orthogonal group of dimension $d$. 

Our results involve the singular value decomposition (SVD) of a matrix $X\in \Md$. We denote this SVD by
\begin{equation}
    \label{eq:newsvd1}
    X = Q_N \Sigma Q_0^T, \quad Q_0,Q_N \in O_d.
\end{equation}
An SVD is unique upto the permutation of singular values. We fix the permutation (so that the SVD is unique) by placing the singular values in decreasing order
\begin{equation}
    \label{eq:newsvd1b}
    \Sigma=\mathrm{diag}\left( \sigma_1, \ldots, \sigma_d\right), \quad \sigma_1 \geq \sigma_2 \geq \ldots \geq \sigma_d. 
\end{equation}

\subsection{The model}
\label{subsec:model}


The parameter space is $\Md^N$. We denote points in $\Md^N$ by
\begin{equation}
    \label{eq:state-space}
    \mathbf{W}= (W_N,W_{N-1}, \ldots, W_1).
\end{equation}
The reverse ordering corresponds to the construction of a function $f: \R^d\to \R^d$ through the composition of functions $f_N \circ f_{N-1} \circ \cdots \circ f_1$ along the layers of a neural network. We restrict to linear functions 
$f_p(x) = W_px$ so that $f(x)= Xx$ where 
\begin{equation}
    \label{eq:state-space2}
    X = W_N W_{N-1}\cdots W_1 =:\phi(\mathbf{W}).
\end{equation}
The matrix $X$ is called the {\em end-to-end\/} matrix in the computer science literature. The notation $X= \phi(\mathbf{W})$ is useful for a geometric description of the dynamics. 


We equip $\Md$ with the Frobenius norm. Thus, the product space $\Md^N$  is a Euclidean space with the norm 
\begin{equation}
    \label{eq:frobenius}
    \|\mathbf{W}\|^2 := \sum_{p=1}^N \Tr \left(W_p^T W_p\right).
\end{equation}

Given a loss function $E: \Md\to \R$, we study the training dynamics described by the gradient flow in $\Md^N$ with respect to the Frobenius norm of the `lifted' loss function $L = E\circ \phi$. This abstract gradient flow, denoted
\begin{equation}
    \label{eq:big-grad1}
    \dot{\mathbf{W}}= - \nabla_{\mathbf{W}} L(\mathbf{W}),
\end{equation}
is in actuality a system of $N$ distinct $\Md$-valued equations. The $p$-th matrix $W_p$ in the network $(W_N,W_{N-1}, \ldots, W_1)$ satisfies
\begin{equation}
    \label{eq:big-grad2}
    \dot{W}_p= - (W_N\cdots W_{p+1})^T dE(X) (W_{p-1}\cdots W_1)^T, \quad p=1,\ldots,N.
\end{equation}
Here $dE(X)$ denotes the $d\times d$ matrix with entries
\begin{equation}
    \label{eq:big-grad3}
    dE(X)_{jk} = \frac{\partial E}{\partial X_{jk}}, \quad 1\leq j,k \leq d.
\end{equation}


\subsection{Invariant varieties and the balanced manifold}
\label{subsec:balance}
The gradient flow~\eqref{eq:big-grad2} has the striking feature that the  parameter space $\Md^N$ is stratified by algebraic varieties that are invariant  under the flow of~\eqref{eq:big-grad1} for all $E \in C^2$. 
Each such variety is a conic section in $\Md^N$ parametrized by 
\begin{equation}
    \label{eq:G-def}
    \mathbf{G}=(G_{N-1},\cdots, G_1) \in \symm_d^{N-1}.
\end{equation}
Given $\mathbf{G}$, define the system of $\symm_d$-valued quadratic equations 
\begin{equation}
    \label{eq:b1}
    W_{p+1}^TW_{p+1} = W_pW_p^T -G_p , \quad p=1,\ldots,N,
\end{equation}
and the  $\mathbf{G}$-{\em balanced\/} variety
\begin{equation}
    \label{eq:b2}
    \balance_{\mathbf{G}} = \{ \mathbf{W}\in \Md^N \left| W_{p+1}^TW_{p+1} = W_pW_p^T  - G_p, \; p=1,\ldots,N \right. \}.
\end{equation}
The special case $\mathbf{G}=\mathbf{0}$ yields the {\em balanced variety\/}
\begin{equation}
    \label{eq:b7}
    \balance_{\mathbf{0}} = \{ \mathbf{W}\in \Md^N \left|  W_{p+1}^TW_{p+1}= W_pW_p^T , \; p=1,\ldots,N \right. \}. 
\end{equation}
If $\ww \in \balance_{\mathbf{0}}$, the singular values, and thus rank, of the $W_p$ are identical. Thus, $\balance_{\mathbf{0}}$ is itself stratified by varieties $\balance_r$ corresponding to the rank $r$, $1\leq r \leq d$. We refer to the leaf of $\balance_{\mathbf{0}}$ with $\mathrm{rank}(W_p)=d$ as the {\em balanced manifold\/}, $\balance$. 

The dimension of $\Md^N$ is $Nd^2$ and equation~\eqref{eq:b1} provides $(N-1)d(d+1)/2$ constraints when $X$ has rank $d$. Thus, 
\begin{equation}
    \label{eq:dimension-balance} \mathrm{dim}(\balance) = d^2 + (N-1)\frac{d(d-1)}{2}.
\end{equation}
Of these, $d^2$ degrees of freedom correspond to $X$. The remaining $(N-1)d(d-1)/2$ degrees of freedom correspond to an $O_d^{N-1}$ action studied in Section~\ref{subsec:orbit}.

\subsection{Previous results}
\label{subsec:prior}
We now recall three results that have inspired our work. Assume that $E\in C^2$ and consider the initial value problem
\begin{equation}
    \label{eq:ivp}
    \dot{\mathbf{W}} = -\nabla_{\mathbf{W}} E\circ 
    \phi (W), \quad \mathbf{W}(0)=\mathbf{W}_0.
\end{equation}
This is a differential equation with a locally Lipschitz vector field. Thus, Picard's theorem guarantees the existence of a unique solution on a maximal time interval $(T_{\min},T_{\max})$ containing $t=0$. Let $\mathbf{G}$ be given by the initial condition
\begin{equation} 
\label{eq:newG}
G_p = W_{p+1}^T(0) W_{p+1}(0) - W_{p}(0) W_{p}^T(0), \quad 1\leq p \leq n. 
\end{equation}
\begin{theorem}[Arora, Cohen, Hazan~\cite{ACH}]
\label{thm:ACH}
The following hold on the maximal interval of existence of solutions to equation~\eqref{eq:ivp}.
\begin{enumerate}   
\item[(a)] The solution $\mathbf{W}(t)$ lies on the variety $\balance_{\mathbf{G}}$.
\item[(b)] The end-to-end matrix $X(t)$ satisfies 
\begin{equation}
    \label{eq:closed0}
    \dot{X} = - \sum_{p=1}^N (A_{p+1}A_{p+1}^T) \,dE(X) \,  (B_{p-1}^T B_{p-1}), 
\end{equation}
where $A_{N+1}=B_0=1$ and we have defined the partial products
\begin{equation}
    \label{eq:closed0b}
    A_p = W_N \cdots W_{p}, \quad B_p = W_p \cdots W_1, \quad 1 \leq p \leq N.
\end{equation}
\end{enumerate}
\end{theorem}

Theorem~\ref{thm:ACH} does not provide a closed description of the reduced dynamics on $\balance_{\mathbf{G}}$. However, on the balanced manifold $\balance$ we have 
\begin{theorem}[Arora, Cohen, Hazan~\cite{ACH}]
    \label{thm:ACH2}
Assume $\mathbf{W}(0) \in \balance$. The end-to-end matrix $X(t) =\phi(\mathbf{W}(t))$ satisfies the differential equation
\begin{equation}
    \label{eq:closed1}
    \dot{X} = - \sum_{p=1}^N(XX^T)^{\tfrac{N-p}{N}}\, dE(X) \, (X^TX)^{\tfrac{p-1}{N}}.
\end{equation}
\end{theorem}
A further surprise is that equation~\eqref{eq:closed1} is a {\em Riemannian\/} gradient flow. Let us state this assertion precisely. 

Fix $1 \leq r \leq d$ and consider the differentiable manifold $\Mr$ of rank-$r$ matrices in $\Md$. Given $X\in \Mr$ consider the linear map
\begin{equation}
    \label{eq:metric2}
    \mathcal{A}_{N,X}: T_X\Mr^* \to T_X\Mr, \quad P \mapsto \sum_{p=1}^N(XX^T)^{\tfrac{N-p}{N}}\, P \,(X^TX)^{\tfrac{p-1}{N}}.
\end{equation}
Now use $\mathcal{A}_{N,X}$ to define the length of $Z \in T_X \Mr$ as follows:
\begin{equation}
    \label{eq:metric3} \|Z\|_{g^N}^2 = g^N(X)(Z,Z) := \Tr(Z^T \mathcal{A}_{N,X}^{-1} Z).
\end{equation}
 We define $(\Mr,g^N)$ as the Riemannian manifold obtained by equipping $\Mr$ with the metric $g^N$. The notation $P$ and $Z$ are used to distinguish the fact that $P \in T_X\Mr^*$ is a $1$-form and $Z\in T_X\Mr$ is a tangent vector. Both are matrices since we have a global coordinate system on $\Md$, but this distinction is important for the following
 
 \begin{theorem}[Bah, Rauhut, Terstiege, Westdickenberg~\cite{Bah}]
\label{thm:BRTW2}
Equation~\eqref{eq:closed1} is equivalent to the Riemannian gradient flow on $(\Mr,g^N)$ 
\begin{equation}
    \label{eq:closed3}
    \dot{X} = - \grad_{g^N} E(X).
\end{equation}
In particular, when $\ww(0)\in \balance_r$, the end-to-end matrix $X(t)=\phi(\ww(t))$ evolves in $\Mr$ according to this Riemannian gradient flow.
\end{theorem}

\section{Statement of results}
\label{sec:results}
\subsection{Overview}
The main new results are Theorem~\ref{thm:entropy} (the entropy formula) and Theorem~\ref{thm:submersion} (Riemannian submersion). This section contains the statements of these results along with some background; the theorems are proved in the sections that follow. The reader interested in the application of these results to training dynamics in deep learning is advised to skip the proofs of these theorems on a first reading and turn to the discussion in Section~\ref{sec:discussion}.

\subsection{Fibers, group orbits and the Boltzmann entropy}
The space $\Md^N$ also admits a complementary stratification by fibers defined as follows .\begin{equation}
    \label{eq:fiber1}
    \fiberX := \phi^{-1}(X) = \{ \mathbf{W}\in \Md^N \left| W_N W_{N-1}\cdots W_1 = X\right. \}.
\end{equation}
We focus our attention on 
\begin{equation}
    \label{eq:group-orbit}
    \orbitx := \fiberX \cap \balance.
\end{equation}
Each point $\mathbf{W} \in \orbitx$ is a point in parameter space $\Md^N$ that generates 
the linear function $f(x)=Xx$ and is consistent with the conservation laws for training dynamics stated in Theorem~\ref{thm:ACH}. In analogy with statistical mechanics, we view the points $\mathbf{W} \in \orbitx$ as `microstates' and the matrix $X$ as an `observable'. 

We show with an explicit parametrization that $\orbitx$ is an $\od^{N-1}$ orbit (see Corollary~\ref{cor:orbitx}). 
Thus, $\orbitx$ is a compact manifold with finite $n_*$-dimensional volume, where 
\begin{equation}
    \label{eq:group-orbit2}
    \nstar  =\mathrm{dim}(O_d^{N-1}) = (N-1) d(d-1)/2.
\end{equation}
The {\em Boltzmann entropy\/} is obtained by enumerating microstates. In our context it is defined by
\begin{equation}
    \label{eq:entropy1}
    S(X)= \log \mathrm{vol}_{\nstar} (\orbitx),
\end{equation}
where the volume form is with respect to the metric inherited by $\balance$ because of its embedding in $\Md^N$. We compute this volume form, and thus $S(X)$, exactly using methods from random matrix theory (see Theorem~\ref{thm:entropy}). In the companion paper~\cite{MY-RLE} we introduce a Riemannian Langevin equation (RLE) for the DLN,  which is a natural stochastic perturbation of the gradient flow~\eqref{eq:big-grad1}. This construction is based on an analogy with Dyson Brownian motion and provides microscopic stochastic dynamics that are consistent with the entropy formula below.

\subsection{The entropy formula}
We denote the Vandermonde determinant associated to a diagonal matrix $A = \mathrm{diag}(\alpha_1,\ldots, \alpha_n)$ by
\begin{equation}
    \label{eq:van1} 
    \mathrm{van}(A) = \det \left(\begin{array}{llll} 1 & 1 & \ldots & 1 \\ \alpha_1 & \alpha_2 & \ldots & \alpha_d \\ \vdots & \vdots & & \vdots \\ \alpha_1^{d-1} & \alpha_2^{d-1} & \ldots & \alpha_d^{d-1} \end{array}\right)
    = \prod_{1 \leq i < j \leq d} (\alpha_j -\alpha_i).
\end{equation}
\begin{theorem}\label{thm:entropy}
    Assume $X$ has full rank and distinct singular values.  Then 
    \begin{eqnarray}
\label{eq:entropy3}    
\lefteqn{ S(X) = (N-1) \log c_d + \frac{1}{2} \log \frac{\mathrm{van}(\Sigma^2)}{\mathrm{van}(\Sigma^{\frac 2N})}}\\
\label{eq:entropy4}
&& = (N-1) \log c_d + \frac{1}{2} \sum_{1\leq i<j\leq d}\log \left( \frac{\sigma_i^{2}-\sigma_j^{2}}{\sigma_i^{\frac2N}-\sigma_j^{\frac2N}} \right), 
     \end{eqnarray}
    where $c_d$ is the volume of the orthogonal group with Haar measure. 
\end{theorem}
\begin{remark}
When $X$ has full rank but repeated singular values, $S(X)$ may be evaluated by taking limits in the above formula using L'Hospital's rule. The value of the constant $c_d$ is
$2^{\frac12d(d+3)}\prod_{r=1}^d\frac{\pi^{\frac r2}}{\Gamma(\frac r2)}$~\cite{Ponting}. 
\end{remark}
\begin{remark}
 The entropy formula extends naturally to the infinite depth ($N\to \infty$) limit. In this setting, the orbit $\orbitx$, and thus its volume, are not defined. However, we may renormalize  equation~\eqref{eq:entropy1} by subtracting divergent terms that are linear in $N$ and $\log N$   to obtain the renormalized entropy
 \begin{equation}
     \label{eq:entropy5} S_\infty(X) = \frac{1}{2} \log \frac{\mathrm{van}(\Sigma^2)}{\mathrm{van}(\log \Sigma^{2})}.
 \end{equation}
\end{remark}
\begin{remark}
In the companion paper~\cite{LM1}, we show that $\mathbf{G}$ are obtained from a moment map on $\fiberX$ associated to the action of the subgroup $\od^{N-1} \subset GL(d)^{N-1}$. This approach yields the general setting of the entropy, with
\begin{equation}
    \label{eq:entropy6}
\orbitxG := \fiberX \cap \balanceG, \quad  S(X;\mathbf{G})=  \log \mathrm{vol}_{n*} \left(\orbitxG \right).
\end{equation}
We recover  equation~\eqref{eq:entropy1} when $\mathbf{G}=\mathbf{0}$. However, it is more subtle to establish an analogue of Theorem~\ref{thm:entropy} when $\mathbf{G}$ is not $\mathbf{0}$. Our proof of Theorem~\ref{thm:entropy} relies crucially on the fact that $\orbitx$ is an $\od^{N-1}$ orbit. But in general $\fiberX\cap \balance_{\mathbf{G}}$ is {\em not\/} an $\od^{N-1}$ orbit. Thus, while it is conceptually clear that $S(X;\mathbf{G})$ is a natural Boltzmann entropy, new ideas are required  to evaluate $S(X;\mathbf{G})$. 
\end{remark}

\subsection{Riemannian submersion}
\label{subsec:Riemannian}
The metric $g^N$ was first introduced in~\cite{Bah} using the linear operator $\mathcal{A}_{N,X}$ defined in equations~\eqref{eq:metric2}--\eqref{eq:metric3}. On the other hand, the balanced manifold $\balance$ inherits a Riemannian metric from its embedding in $\Md^N$; we denote the resulting Riemannian manifold by $(\balance,\iota)$. We prove

\begin{theorem}
    \label{thm:submersion}
The map $\phi: \balance \to \Mdd$ is a Riemannian submersion from $(\balance,\iota)$ to $(\Mdd,g^N)$.    
\end{theorem}
Theorem~\ref{thm:submersion} provides the natural geometric explanation for certain unexpected properties of the metric $g^N$ observed in the literature. Let us motivate Theorem~\ref{thm:submersion} by recalling two such properties: the diagonalization of $\aaa_{N,X}$ using the SVD and a determinantal formula related to Theorem~\ref{thm:entropy}~\cite{CMV}. 

Let $\{q_{N,1},\cdots, q_{N,d}\}$ and $\{q_{0,1},\cdots, q_{0,d}\}$ be the column vectors of $Q_N$ and $Q_0$ respectively. Then we may rewrite the SVD~\eqref{eq:newsvd1} as
\begin{equation}
    \label{eq:diagon1}
    X = Q_N \Sigma Q_0^T = \sum_{j=1}^d \sigma_j q_{N,j} q_{0,j}^T. 
\end{equation} 
\begin{lemma}
    \label{le:diagon}
    The operator $\aaa_{N,X}: T_X \Mdd^* \to T_X \Mdd$ is symmetric and positive definite with respect to the Frobenius inner-product. It has the spectral decomposition
    \begin{equation}
        \label{eq:diagon2}
        \aaa_{N,X} \, q_{N,k} q_{0,l}^T = \frac{\sigma_k^2-\sigma_l^2}{\sigma_k^{\frac{2}{N}}- \sigma_l^{\frac{2}{N}}} \, q_{N,k} q_{0,l}^T , \quad 1\leq k, l \leq d,
    \end{equation}
    when $k \neq l$ and 
\begin{equation}
        \label{eq:diagon2b}
        \aaa_{N,X} \, q_{N,k} q_{0,k}^T = N \sigma_k^{2-\tfrac{2}{N}}\, q_{N,k} q_{0,k}^T , \quad 1\leq k \leq d.
    \end{equation}
    \end{lemma}
The proof is a direct calculation (see~\cite[\S 7]{GM-dln} for details). 
Every matrix $Z \in T_X\Mdd$ may be expressed in this basis as a sum $Z= \sum_{k,l=1}^d Z_{kl} q_{N,k} q_{0,l}^T$. Equation~\eqref{eq:metric3} and Lemma~\ref{le:diagon} then imply
\begin{equation}
    \label{eq:p-metric1}
       g^N(Z,Z) =  \sum_{1 \leq k \leq d} \frac{Z_{kk}^2}{N\sigma_k^{2-\tfrac{2}{N}}}  +     
       \sum_{1\leq k, l \leq d, k \neq l} 
       \frac{\sigma_k^{\frac{2}{N}}- \sigma_l^{\frac{2}{N}}}{\sigma_k^2-\sigma_l^2}
       Z_{kl}^2.
\end{equation}
Lemma~\ref{le:diagon} also allows us to compute the volume form associated to $g^N$ as noted in 
\begin{theorem}[Cohen, Menon, Veraszto~\cite{CMV}]
\label{thm:cmv}
 The volume form on $(\Mdd,g^N)$ is 
 \begin{equation}
     \label{eq:volume-form}
     \sqrt{\det g^N}\, dX = \det(\Sigma^2)^{\frac{N-1}{2N}}\mathrm{van}(\Sigma^{\frac{2}{N}}) \, d\Sigma \,dQ_0\, dQ_N.
 \end{equation}
\end{theorem}
Both Theorem~\ref{thm:entropy} and Theorem~\ref{thm:cmv} are determinantal formulas. However, Theorem~\ref{thm:cmv} concerns $(\Mdd,g^N)$ `downstairs' and follows easily from Lemma~\ref{le:diagon} (we compute $\det g^N$ by multiplying the $d^2$ eigenvalues of $\aaa_{N,X}$). On the other hand, Theorem~\ref{thm:entropy} concerns the Riemannian manifold $(\balance,\iota)$ `upstairs' and is harder to establish. The main step is the construction of a global orthonormal basis on $(\balance,\iota)$. This relies on an explicit parametrization of $\orbitx$ as a group orbit and the construction of an an orthonormal basis using the theory of Jacobi matrices. The computation of this orthonormal basis is the main technical step in this paper and underlies Theorem~\ref{thm:entropy} and Theorem~\ref{thm:submersion}. 

\subsection{Organization of proofs}
\label{subsec:organization}
Theorem~\ref{thm:entropy} and Theorem~\ref{thm:submersion} are proved as follows. We begin with a rigorous description of the group actions corresponding to overparametrization in Section~\ref{sec:group-action}. We then establish the entropy formula in Section~\ref{sec:orbit-basis}. The main step is an explicit description of the Riemannian geometry of $\orbitx$. These calculations are then extended to the Riemannian geometry of the balanced manifold $\balance$ in Section~\ref{sec:balance-basis}. The orthonormal basis for $T_\bfW\balance$ is presented in Theorem~\ref{thm:orthonormal-basis} in Section~\ref{sec:balance-basis}. Theorem~\ref{thm:submersion} on Riemannian submersion follows from directly Theorem~\ref{thm:orthonormal-basis}. 


\section{Group actions and parametrization of $\balance$}
\label{sec:group-action}
\subsection{Parametrizing the balanced variety $\balance_0$}
\label{subsec:parametrization}
The balanced variety $\balance_0$ is the solution set for
\begin{equation}
    \label{eq:b1a}
    W_{p+1}^TW_{p+1} = W_pW_p^T, \quad p=1,\ldots,N.
\end{equation}
It is immediate that the singular values of $W_p$ are the same. A deeper property of the balanced variety is that equation~\eqref{eq:b1a} describes an alignment of the right singular vectors of $W_{p+1}$ and left singular values of $W_p$ along the network. 

We use this property to construct a covering map for the balanced manifold $\balance$ as follows. Define the parametrization
\begin{equation}
    \label{eq:b10} \mathfrak{z}: \R_+^d \times O_d^{N+1} \to \balance_0, \qquad (\Lambda, Q_N, \ldots, Q_0) \mapsto (W_N,\ldots,W_1),
\end{equation}
where we set
\begin{equation}
\label{eq:b11}
W_p = Q_p \Lambda Q_{p-1}^T, \quad 1 \leq p \leq N.
\end{equation}
The relation between $\Lambda$ and the SVD of $X$ (see equation~\eqref{eq:newsvd1}) is
\begin{equation}
\label{eq:b11a}
X= Q_N \Sigma Q_0^T, \quad \Lambda^N = \Sigma.
\end{equation}
This parametrization may be extended to the case when $\mathrm{rank}(X)=r$. In this case, $\sigma_{r+1}=\ldots=\sigma_d=0$ and we must replace $O_d$ by the Stiefel manifold of rank-$r$. However, in order to illustrate the main ideas without any complications, we restrict ourselves in this paper to the balanced manifold $\balance \subset \balance_0$.
In this setting, $X$ has full rank and we have the following
\begin{theorem}\label{thm:parameter}
The image of the mapping $\mathfrak{z}$ is $\balance_0$. Further, if $X$ has full rank and its singular values are distinct, then $\mathfrak{z}$ is locally an analytic diffeomorphism between an open neighborhood of $\balance$ and an open neighborhood in $\R_+^d \times O_d^{N+1}$. 
\end{theorem}
\begin{proof}
Assume $\mathbf{W} =(W_N,\cdots, W_1) \in\mathcal M_0$.  Each $W_p$ has an SVD that we may denote $W_p=Q_p \Lambda_p \tilde{Q}_p^T$. The condition~\eqref{eq:b1a} requires that 
    \begin{align}
        Q_p \Lambda_p^2 Q_p^T = \tilde{Q}_{p+1}\Lambda_{p+1}^2 \tilde{Q}_{p+1}^T, \quad p=1,\ldots,N-1\,.
    \end{align}
It is clear that the choice $\Lambda_p=\Lambda_{p+1}$ and $\tilde Q_{p+1}=Q_p$ for all $p=1,\ldots,N-1$ provides a solution to equations ~\eqref{eq:b1a}. We further choose $Q_N$ and $Q_0$ according to the SVD of $X$. It follows that $\mathbf{W} = \mathfrak{z}(\Lambda,Q_N,\cdots,Q_0)$.

This proves that the image of $\mathfrak{z}$ is $\balance_0$. On the other hand, each $\mathbf{W}\in \balance_0$ may admit many preimages because an SVD may not be unique. The causes of non-uniqueness are rank-deficiency and non-uniqueness of the singular values. Indeed, when $X$ has rank-$r$, $\sigma_{r+1}=\cdots \sigma_d=0$, and the associated singular vectors for each $W_p$ are not unique. When $X$ has full rank, but we have repeated singular values, then $\mathbf{W}\in \balance$ but again the singular vectors associated to the repeated singular values are not unique.

When $X$ has full rank and the singular values are distinct, then each $W_p$ also has full rank and distinct singular values. Now the SVD of each $W_p$ is unique except upto a permutation and the singular vectors depend analytically on $W_p$ by the standard perturbation theory for eigenvalues~\cite{Kato}. Thus, $\mathfrak{z}$ is locally an analytic diffeomorphism. 
\end{proof}

\subsection{$\orbitx$ is an $O_d^{N-1}$ orbit}
\label{subsec:orbit}
Theorem~\ref{thm:parameter} has the following 
\begin{corollary}
\label{cor:orbitx}
Assume $X$ has full rank and distinct singular values. Then $\orbitx$ is an $O_d^{N-1}$ orbit.    
\end{corollary}
\begin{proof}
The orthogonal matrices $Q_N$ and $Q_0$ are fixed by the SVD of $X$. The image of $(Q_{N-1},\cdots, Q_1) \in O_d^{N-1}$ under $\mathfrak{z}$ is $\orbitx$ by the construction of Theorem~\ref{thm:parameter}.
\end{proof}
We use this Corollary to describe the tangent space $T_{\bfW}\orbitx$ at an arbitrary point $\bfW \in \orbitx$ in terms of the tangent space $T_{\bfC}\orbitx$ at a special point $\bfC$ that we term the {\em center\/} of $\fiberX$.  

Given $X$, consider its SVD given in equation~\eqref{eq:newsvd1}, and the related diagonal matrix $\Lambda = \Sigma^{\tfrac{1}{N}}$. We define the center of $\fiber_\Sigma$ to be the diagonal matrices 
\begin{equation}
\label{eq:newsvd2}
\bfC_\Sigma = (\Lambda, \cdots, \Lambda) \in \fiber_\Sigma,
\end{equation}
In a similar manner, we define the center of $\fiberX$ as
\begin{equation}
\label{eq:newsvd3}
\bfC_X = (Q_N\Lambda, \cdots, \Lambda Q_0^T) \in \fiberX.
\end{equation}
When the base point ($\Sigma$ or $X$) is clear, we drop the subscript on $\bfC$.

The group $GL(d)^{N-1}$ acts naturally on $\fiberX$. Let $\bfA =(A_{N-1},A_{N-2}, \cdots, A_1) \in GL(d)^{N-1}$ and define the action 
\begin{equation}
\label{eq:action1}
\fiberX \to \fiberX, \quad \bfW \mapsto \bfA \cdot \bfW := (W_N A_{N-1}^{-1},A_{N-1} W_{N-1} A_{N-2}^{-1},\cdots, A_1 W_1).   
\end{equation}
We now restrict attention to $\orbitx$ and $\bfA=\bfQ \in O_d^{N-1}$. We may express each point $\bfW\in \orbitx$ as $\bfQ  \cdot \bfC_X$ for a suitable $\bfQ \in O_d^{N-1}$. Indeed, we only need to observe that if $\bfQ =(Q_{N-1},\ldots, Q_1) \in O_d^{N-1}$ then 
\begin{equation}
\label{eq:newsvd9a}
\bfQ\cdot \bfC_X = (Q_N \Lambda Q_{N-1}^T, Q_{N-1} \Lambda Q_{N-2}^T,\ldots, Q_1 \Lambda Q_0^T).
\end{equation}
Finally, many Riemannian calculations are simplified by the following
\begin{lemma}
    \label{le:isometry}
    The $O_d^{N-1}$ action $\bfW \mapsto \bfQ\cdot \bfW$ is an isometry of $\fiberX$.
\end{lemma}
\begin{proof}
The Frobenius norm on $\Md^N$ is invariant under the transformation $\bfW \mapsto \bfQ\cdot \bfW$. Since $\fiberX$ inherits its metric from $\Md^N$ and is invariant under the group action $\bfW \mapsto \bfQ\cdot \bfW$, this transformation is also an isometry of $\fiberX$.
\end{proof}
We use Lemma~\ref{le:isometry} to reduce the computation of an orthonormal basis for $T_{\bfW} \orbitx$ to the computation of an orthonormal basis for $T_{\bfC} \orbitx$. Further, we may reduce the analysis to the situation when $X=\Sigma$ as follows. While the $O_d^{N-1}$ action preserves $\fiberX$, our system also admits a natural $O_d^{N+1}$ action that preserves $\balance$. Given $\bfQ^{N}=(Q_N,\cdots,Q_0) \in O_d^{N+1}$ define
\begin{equation}
    \label{eq:big-group1}
    \bfQ^{N}:\bfW = \left(Q_N W_NQ_{N-1}^T, \ldots, Q_1 W_1 Q_0^T \right).
\end{equation}
\begin{lemma}
    \label{le:isometry2}
    The $O_d^{N+1}$ action $\bfW \mapsto \bfQ^{N}:\bfW$ is an isometry of $\balance$.
\end{lemma}
\begin{corollary}
\label{cor:isometry-entropy} $S(X)=S(\Sigma)$.
\end{corollary}
The proof of Lemma~\ref{le:isometry2} is almost identical to that of Lemma~\ref{le:isometry} and is omitted. It allows us to reduce calculations on $\fiberX$ to $\fiber_\Sigma$.

The reader may also use Lemma~\ref{le:diagon} to show that the map $X \mapsto Q_N X Q_0^T$ is an isometry of $(\mathfrak{M}_d,g^N)$. This result can be obtained without explicitly using the spectral decomposition of $g^N$ using Theorem~\ref{thm:submersion} and Lemma~\ref{le:isometry2}.

\subsection{The tangent spaces $T_\bfW\fiberX$ and $T_\bfW \orbitx$}
\label{subsec:orbit-tangent}
We use the $GL(d)^{N-1}$ action on $\fiberX$ to compute the tangent space $T_{\bfW}\fiberX$ as follows. Recall that the Lie algebra $gl(d)^{N-1}$ is the vector space $\Md^{N-1}$, with elements denoted $\bfa=(a_{N-1},\ldots,a_1)$, equipped with the Lie bracket $[a_p,b_p]=a_pb_p-b_pa_p$ for each component $p=1,\ldots,N$. Each $\bfa \in gl(d)^{N-1}$ defines a curve through the identity via
\begin{equation}
    \label{eq:tangent1}
    \bfA(t) = (e^{ta_{N-1}}, \cdots, e^{ta_1}) := e^{t \bfa}, \quad t \in (-\infty,\infty).
\end{equation}
The tangent space $T_{\bfW}\fiberX$ consists of the vectors 
\begin{equation}
    \label{eq:tangent2}
    \bfw_\bfa := \left. \frac{d}{dt} e^{t \bfa} \cdot \bfW\right|_{t=0}, \quad \bfa \in \Md^{N-1}.
\end{equation}
We substitute in equation~\eqref{eq:action1} to find that 
\begin{equation}
    \label{eq:tangent3}
    \bfw_\bfa = (-W_N a_{N-1}, a_{N-1}W_{N-1} - W_{N-1}a_{N-2}, \cdots, a_1 W_1), \quad \bfa \in \Md^{N-1}.
\end{equation}
In particular, when $X=\Sigma$ and $\bfW=\bfC_\Sigma$, we find 
\begin{lemma}
    \label{le:tangent-center}
    The tangent space $T_{\bfC_\Sigma}\orbit_\Sigma$ consists of 
    \begin{equation}
        \label{eq:tangent-center1}
        \bfc_\bfa = (-\Lambda a_{N-1}, a_{N-1} \Lambda - \Lambda a_{N-2}, \cdots, a_1 \Lambda), \quad \bfa \in \anti_d^{N-1}.
    \end{equation}
\end{lemma}
\begin{proof}
We use Corollary~\ref{cor:isometry-entropy} and apply equation~\eqref{eq:tangent3} with $\bfW=\bfC_\Sigma$. Note that $\bfa \in \anti_d^{N-1}$ since we now restrict attention to the action of the subgroup $O_d^{N-1}$.
\end{proof}
A similar calculation applies at each $\bfW \in \orbitx$. We may simplify our calculations further using equation~\eqref{eq:action1}.
\begin{lemma}
    \label{le:tangent-center2}
    Assume $\bfW = \bfQ \cdot \bfC_X \in \orbitx$.
    The tangent space $T_{\bfW}\orbitx$ consists of
    \begin{equation}
        \label{eq:tangent-center3}
        \bfQ \cdot \bfc  := \left(-Q_N c_N Q_{N-1}^T, Q_{N-1} c_{N-1} Q_{N-2}^T, \cdots, Q_1 c_1 Q_0^T\right), \quad \bfc \in T_{\bfC_\Sigma} \orbit_\Sigma.
    \end{equation}
\end{lemma}
We use Lemma~\ref{le:isometry} and Lemma~\ref{le:tangent-center} to reduce the computation of an orthonormal basis for $T_\bfW \orbitx$ to the computation of an orthonormal basis for $T_{\bfC_\Sigma} \orbit_\Sigma$.

\section{Proof of the entropy formula}
\label{sec:orbit-basis}
\subsection{Overview}
In this section we use the $O_d^{N-1}$ action on $(\orbitx,\iota)$ to compute an orthonormal basis for $(T_{\bfC_X}\orbitx,\iota)$. To this end, we first fix $\Sigma$,$Q_N$ and $Q_0$ and let $\mathfrak{y}$ denote the restriction of $\mathfrak{z}$ so obtained (see equation~\eqref{eq:b10}). We compute the pullback metric $\mathfrak{y}^\sharp \iota$ in Corollary~\ref{cor:metric-orbit} below. Then we diagonalize the pullback metric using the theory of Chebyshev polynomials, obtaining Theorem~\ref{thm:entropy} as a consequence. 

We assume throughout that $X$ has full rank. Our calculations have natural modifications to the case when $X$ is rank-deficient since the main step is the computation of the pullback metric and its diagonalization. We restrict ourselves to full rank matrices to illustrate the main ideas without technicalities.

\subsection{Proof of Theorem~\ref{thm:entropy}}
\label{subsec:orbit-basis}
We state several lemmas in this subsection, concluding with a proof of Theorem~\ref{thm:entropy}. The lemmas are proved in the subsection that follows. 

Recall Lemma~\ref{le:tangent-center}. Let $\{e_1, \ldots, e_d\}$ denote the standard orthonormal basis for $\R^d$ and construct the standard orthonormal basis for $\anti_d$
\begin{equation}
    \label{eq:ad-basis1}
    \alpha^{k,l} := \frac{1}{\sqrt{2}}\left( e_k e_l^T - e_l e_k^T\right), \quad 1 \leq k < l \leq d.
\end{equation}
We then define the standard orthonormal basis for $\anti_d^{N-1}$
\begin{equation}
    \label{eq:ad-basis2}
    \bfa^{k,l,p} := \left(0, \cdots, \alpha^{k,l}, \cdots,0\right), \quad 1 \leq k < l \leq d, \quad 1 \leq p \leq N-1,
\end{equation}
where the matrix $ \alpha^{k,l}$ appears at depth $p$. The push-forward of this basis vector under the tangent map of equation~\eqref{eq:tangent2} is denoted
\begin{equation}
        \label{eq:ad-basis3}
        \bfc^{k,l,p} = \left(0, \cdots,-\Lambda \alpha^{k,l}, \alpha^{k,l}\Lambda, \cdots,0\right).  
    \end{equation}
The non-zero entries appear at depth $p+1$ and $p$. Lemma~\ref{le:tangent-center} implies 
\begin{lemma}
    \label{le:basis1}
    The  vectors $\{\bfc^{k,l,p}\}_{1\leq p \leq N-1, 1 \leq k < l \leq d}$ form a basis for $T_C\orbitx$.
\end{lemma}
We denote the Frobenius inner product between $\bfv=(v_N,\cdots, v_1)$ and $\bfw=(w_N,\cdots,w_1)$ by
\begin{equation}
    \label{eq:ad-basis4}
    \langle \bfv, \bfw \rangle = \sum_{p=1}^N \Tr(v_p^T w_p). 
\end{equation}
In order to construct an orthonormal basis for $T_C\orbitx$, we must compute the inner-product between these basis vectors and diagonalize the resulting matrix. Our calculation is simplified by the following consequence of the sparsity of $\bfc^{k,l,p}$.
\begin{lemma}
    \label{le:basis2}
The inner-product    $\langle \bfc^{k,l,p}, \bfc^{m,n,q} \rangle =0$  if $|p-q|>1$.
\end{lemma}
Lemma~\ref{le:basis2} implies that the matrix of inner-products is block tridiagonal when arranged according to the depth index $p$. This structure is further simplified by 
\begin{lemma}
\label{le:basis3}
The inner-product  $\langle \bfc^{k,l,p}, \bfc^{m,n,q} \rangle =0$  if the pair $(kl)$ and $(mn)$ are distinct. The only non-zero inner products are
\begin{equation}
    \label{eq:basis5} \langle \bfc^{k,l,p}, \bfc^{k,l,p} \rangle = \lambda_k^2 + \lambda_l^2
\end{equation}  
\begin{equation}
    \label{eq:basis6} \langle \bfc^{k,l,p}, \bfc^{k,l,p+1} \rangle = -\lambda_k \lambda_l.
\end{equation}  
\end{lemma}
Lemma~\ref{le:isometry} implies that these inner-products remain the same at any point 
$\bfQ\cdot \bfC_X \in \orbitx$ when $\bfc^{k,l,p} \in T_{\bfC_X} \orbitx$ is translated to $\bfQ\cdot \bfc^{k,l,p} \in T_{\bfQ\cdot\bfC_X} \orbitx$. We summarize these lemmas in the following 
\begin{corollary}
    \label{cor:metric-orbit}
    \label{le:tridiag}
    The standard basis on $\anti_d^{N-1}$ may be ordered such that the pullback metric $\mathfrak{y}^\sharp \Fr$ has the block diagonal structure
\begin{equation}
    \label{eq:pb1}
    h_a = \left( \begin{array}{ccc} 
          h_{a}^{1,2} &  & \\
             &  \ddots     &  \\
         &  & h_a^{d-1,d}
    \end{array} \right),
\end{equation}
where $h_a^{k,l}$ is the $(N-1)\times (N-1)$ symmetric tridiagonal matrix
\begin{equation}
    \label{eq:pb2}
    h_a^{k,l} = \left( \begin{array}{ccccc} 
     \lambda_k^2 +\lambda_l^2 & -\lambda_k \lambda_l &  & &\\
               -\lambda_k \lambda_l  & \lambda_k^2 +\lambda_l^2  & -\lambda_k \lambda_l& &
               \\   & \ddots  & \ddots& -\lambda_k \lambda_l \\ &\cdots &     -\lambda_k \lambda_l & \lambda_k^2 +\lambda_l^2  \\
        
    \end{array} \right).
\end{equation}
\end{corollary}
There are $d(d-1)/2$ blocks $h_a^{k,l}$. Each such block $h_a^{k,l}$ has size $(N-1)\times (N-1)$ and corresponds to the basis matrix $\alpha^{k,l} \in \anti_d$.  The decoupling of these blocks follows from Lemma~\ref{le:basis3}. For fixed $(kl)$, we obtain a symmetric tridiagonal (Jacobi) matrix of size $(N-1) \times (N-1)$ whose entries are given by Lemma~\ref{le:basis3}.

The matrix $h_a^{k,l}$ is a modification of the Jacobi matrix corresponding to Chebyshev polynomials and it may be diagonalized using the standard theory. Fix $(kl)$ and define the orthogonal matrix $S \in O_{N-1}$ by
\begin{equation}
    \label{eq:cheby1}
    S_{pq} = \sqrt{\frac{2}{N-1}}\sin \frac{2 pq \pi}{N}, \quad 1\leq p,q \leq N-1 , 
\end{equation}
and the diagonal matrix $\Sigma^{k,l} =\mathrm{diag}(\sigma^{k,l}_1,\ldots,\sigma^{k,l}_d)$ by
\begin{equation}
    \label{eq:cheby2}
   \sigma^{k,l}_p = \lambda_k^2 +\lambda_l^2 - 2\lambda_k \lambda_l \cos \frac{p \pi }{N}, \quad 1\leq p \leq N-1.
\end{equation}

\begin{lemma}
\label{le:basis6}
The matrix $h^{k,l}_a$ may be diagonalized as follows
\begin{equation}
    \label{eq:basis7} S h_a^{k,l} S^T = \Sigma^{k,l}.
\end{equation}  
Further, 
\begin{equation}
    \label{eq:basis8} \det h^{k,l} = \frac{\lambda_k^{2N}-\lambda_l^{2N}}{\lambda_k^2-\lambda_l^2}.
\end{equation}  
\end{lemma}
Since $\det(h_a)$ is block diagonal and $\Lambda^N=\Sigma$, we also have 
\begin{corollary}
    \label{cor:determinant}
    \begin{equation}
        \label{eq:entropy-determinant}
        \det(h_a) = \frac{\mathrm{van}(\Sigma^{2})}{\mathrm{van}(\Sigma^{\tfrac{2}{N}})}.
    \end{equation}
\end{corollary}
\begin{proof}[Proof of Theorem~\ref{thm:entropy}]
Since $X$ has full rank and distinct singular values $\mathfrak{y}: O_d^{N-1} \to \orbitx$ is a diffeomorphism. 

We compute the volume of $\orbitx$ by working with the pullback metric $\mathfrak{y}^\sharp \iota$ on $O_d^{N-1}$. Let $dv_o$ denote the volume form with respect to Haar measure on $O_d^{N-1}$. Then 
\begin{equation}
\label{eq:volume-proof1}
    \mathrm{vol}_{\nstar}(\orbitx)= \int_{O_d^{N-1}} \sqrt{\mathrm{det}_{\nstar}(\mathfrak{y}^\sharp \iota)} \, dv_o .
\end{equation}
We now apply Corollary~\ref{cor:metric-orbit} and Corollary~\ref{cor:determinant} to see that
\begin{equation}
\label{eq:volume-proof2}
    \mathrm{vol}_{\nstar}(\orbitx)=  \sqrt{\frac{\mathrm{van}(\Sigma^{2})}{\mathrm{van}(\Sigma^{\tfrac{2}{N}})}} \int_{O_d^{N-1}}\;dv_0 = c_d^{N-1} \sqrt{\frac{\mathrm{van}(\Sigma^{2})}{\mathrm{van}(\Sigma^{\tfrac{2}{N}})}}.
\end{equation}

\end{proof}

\subsection{Proofs of Lemmas}
\label{subsec:proof-basis1}
\begin{proof}[Proof of Lemma~\ref{le:basis1}]
Since $X$ has full-rank, all entries of $\Lambda$ are strictly positive. Then the linear transformation between $\bfa^{k,l,p}$ and $\bfc^{k,l,p}$ in equations~\eqref{eq:ad-basis1}--~\eqref{eq:ad-basis3} is invertible. Thus, $\anti_d^{N-1}$ and $T_{\bfC_X}\orbitx$ are isomorphic.
\end{proof}

\begin{proof}[Proof of Lemma~\ref{le:basis2}]
Lemma~\ref{le:basis2} follows immediately from equation~\eqref{eq:ad-basis3}: when the depth indices satisfy $|p-q|>1$ there is no overlap between the non-zero entries of $\bfc^{k,l,p}$ and $\bfc^{m,n,q}$.
\end{proof}

\begin{proof}[Proof of Lemma~\ref{le:basis3}]
We must compute the inner-products corresponding to neighboring depths $p$ and $p+1$ 
\begin{equation}
    \label{eq:basis10}
    \langle \bfc^{k,l,p}, \bfc^{m,n,p} \rangle \quad\mathrm{and} \quad \langle \bfc^{k,l,p}, \bfc^{m,n,p+1} \rangle.
\end{equation}
Consider the first inner-product. It follows from the definitions that 
\begin{equation}
    \label{eq:ip1}
    \bfc^{k,l,p} = (0, \ldots, -\Lambda\alpha^{k,l}, \alpha^{k,l}\Lambda,0 ), \quad \bfc^{m,n,p} = (0, \ldots, -\Lambda\alpha^{m,n}, \alpha^{m,n}\Lambda, 0 ).
\end{equation}
We then have (using $(\alpha^{k,l})^T=-\alpha^{k,l}$)
\begin{equation}
    \label{eq:ip2}
   \langle \bfc^{k,l,p}, \bfc^{p,mn}\rangle = -2 \Tr( \alpha^{k,l} \Lambda^2 \alpha^{m,n}) . 
\end{equation}
After a computation using~\eqref{eq:ad-basis1}, we find that 
\begin{equation}
\label{eq:ip3}
-2 \Tr( \alpha^{k,l} \Lambda^2 \alpha^{m,n}) = (\lambda_m^2 +\lambda_n^2)\left( \delta_{km}\delta_{kl}- \delta_{kn}\delta_{lm}\right). 
\end{equation}
Since $k<l$ and $m<n$, the term $\delta_{kn}\delta_{lm}$ is always zero. Similarly, unless $k=m$ and $l=n$ the first term vanishes. This yields equation~\eqref{eq:basis5}.

Next consider the second case in equation~\eqref{eq:basis10}. We find that
\begin{equation}
\label{eq:ip4}
   \langle \bfc^{k,l,p}, \bfc^{p+1,mn}\rangle = - \Tr( \Lambda \alpha^{k,l} \Lambda \alpha^{m,n}). 
\end{equation}
As with equation~\eqref{eq:ip3}, we now find that 
\begin{equation}
\label{eq:ip5}
\Tr( \Lambda \alpha^{k,l} \Lambda \alpha^{m,n})= \lambda_k\lambda_l (\delta_{km}\delta_{ln}-\delta_{kn}\delta_{lm}).
\end{equation}
The second term always vanishes because of the constraint $k<l$ and $m<n$. The first term yields the desired result when $k=m$ and $l=n$. 
\end{proof}

\begin{proof}[Proof of Lemma~\ref{le:basis6}]
The key observation is that $h^{k,l}$ is a translation of the Jacobi matrix corresponding to the Chebyshev polynomials. Indeed, 
\begin{equation}
    \label{eq:ip6}
     h^{k,l} = (\lambda_k-\lambda_l)^2  I_{N-1} + \lambda_k \lambda_l C, \quad{\mathrm{where}}\quad C= \left( \begin{array}{rrrr} 2 & -1 & 0 & \\ -1 & 2 & -1 & \\ & \ddots & \ddots & -1 \\ & & -1 & 2 \end{array}\right).
\end{equation}
The formulas for the eigenvalues and eigenvectors of $C$ are well known and are easily modified to yield the identities~\eqref{eq:cheby1}--\eqref{eq:basis8} in Lemma~\ref{le:basis6} (see for example~\cite[p.476]{GvL}). The reader unfamiliar with this theory may simply verify these identities to complete the proof.
\end{proof}

\section{An orthonormal basis for $(\balance,\iota)$}
\label{sec:balance-basis}
\subsection{Overview}
In this section we extend the computation of the orthonormal basis for $T_{\bfW}\orbitx$ in Section~\ref{sec:orbit-basis}to $T_{\bfW}\balance$. The main new idea is to vary $X$ and to use the SVD coordinates $X=Q_N \Sigma Q_0^T$ so that we may follow the analysis of Section~\ref{sec:orbit-basis} in all essential details. We assume throughout this section that the singular values of $X$ are distinct (see Remark~\ref{rem:svd-distinct-basis}).

\subsection{An orthonormal basis for $T_\bfW\balance$.}
\label{subsec:basis}
We define an orthonormal basis for $T_\bfW\balance$ in this section. This basis is obtained as follows. First, we use the parametrization~\eqref{eq:b10} and the calculations of Section~\ref{subsec:orbit} to see that the pushforwards of the standard basis on $\R^d \times \anti_d^{N+1}$ provide a basis for $T_\bfW\balance$. Next, we orthonormalize this basis using the theory of Chebyshev polynomials. As in Section~\ref{subsec:orbit}, it is enough to understand the orthonormal basis for $\bfW=\bfC$; this basis can be translated to $T_\bfW \balance$ using the isometry of Lemma~\ref{le:tangent-center2}.

Assume that $\ww = \mathfrak{z}(\Lambda,Q_N,\ldots, Q_0)$. We denote the columns of each $Q_p$ by
\begin{equation}
    \label{eq:cheb1}
    Q_p = \left( \begin{array}{cccc}
    \uparrow & \uparrow  & \cdots & \uparrow \\
          q_{p,1} & q_{p,2} & \cdots & q_{p,d} \\
     \downarrow & \downarrow  & \cdots & \downarrow     \end{array}\right).
\end{equation}
The orthonormal basis for $T_\bfW\balance$ consists of a set 
\begin{eqnarray}
    \label{eq:vb1}
    \mathbf{v}^{k} &=& (v_N^{k},\cdots, v_{1}^{k}), \quad 1\leq k \leq d;\\
    \label{eq:vb2}
    \mathbf{u}^{k,l,p} &=& (u_N^{k,l,p},\cdots, u_{1}^{k,l,p}), \quad 1\leq k < l\leq d, \quad 0 \leq p \leq N. 
\end{eqnarray}
There are $d$ vectors of type $\mathbf{v}$ and $(N+1)d(d-1)/2$ vectors of type $\mathbf{u}$. The coordinates of these vectors are as follows. First, for the $\mathbf{v}$ vectors 
\begin{equation}
\label{eq:vb3}
    v_s^{k} = \frac{1}{\sqrt{N}} q_{s,k} q_{s-1,k}^T, \quad 1\leq s \leq N.  
\end{equation}
Both $p$ and $s$ index depth along the network, while $k$ and $l$ index matrix entries. 

We next consider the $\mathbf{u}$ vectors for $1\leq p \leq N-1$. This range for $p$ corresponds to the gauge freedom $\od^{N-1}$. We define
\begin{equation}
\label{eq:vb4}
    u_s^{k,l,p}  =  a^{k,l,p,s} q_{s,k} q_{s-1,l}^T + a^{l,k,p,s} q_{s,l} q_{s-1,k}^T ,   \quad 1\leq s \leq N, 
\end{equation}
where we have defined the constants 
\begin{equation}
\label{eq:vb5}
    a^{k,l,p,s}  = \sqrt{\frac{1}{N\left(\lambda_k^2+\lambda_l^2-2\lambda_k\lambda_l \cos 
    \tfrac{p\pi}{N}\right)}} \left(\lambda_k \sin \frac{(s-1)p\pi}{N}- \lambda_l \sin \frac{sp\pi}{N}\right).  
\end{equation}
Finally, the basis vectors indexed by $p=0$ and $N$ correspond to the action of the matrices $Q_0$ and $Q_N$ respectively. For $p=0$ set
\begin{equation}
\label{eq:vb6}
    u_s^{k,l,0}  = \sqrt{\frac{\lambda_k^2-\lambda_l^2}{\lambda_k^{2N}-\lambda_l^{2N}}} \lambda_k^{s-1}\lambda_l^{N-s} q_{s,l} q_{s-1,k}^T, \quad 1\leq s \leq N.   
\end{equation}
Similarly, for $p=N$ define
\begin{equation}
\label{eq:vb7}
    u_s^{k,l,N}  = \sqrt{\frac{\lambda_k^2-\lambda_l^2}{\lambda_k^{2N}-\lambda_l^{2N}}} \lambda_k^{N-s}\lambda_l^{s-1} q_{s,k} q_{s-1,l}^T, \quad 1\leq s \leq N.   
\end{equation}
We then have 
\begin{theorem}
    \label{thm:orthonormal-basis} The vectors $(\mathbf{v},\mathbf{u})$ defined in equations~\eqref{eq:vb3}--\eqref{eq:vb7} form an orthonormal basis for $(T_\ww\balance,\Fr)$.
\end{theorem}


\subsection{Computing the tangent space $T_\bfW\balance$}
\label{subsec:tangent-balance}
Recall the parametrization $\mathfrak{z}$ of $\balance$ given in equation ~\eqref{eq:b10}--\eqref{eq:b11a}. We may compute the tangent space $T_{\ww}\balance$ by differentiating this parametrization as follows. For each $(\theta,\bfa) \in \R^d \times \anti_d^{N+1}$ we define the smooth curve in $\ww(t) \in \balance$ using 
\begin{equation}
\label{eq:bb1}
\Lambda(t) = \Lambda+ t \theta, \quad Q_p(t) = e^{ta_{p}} Q_p,\quad \bfW(t)=\mathfrak{z}(\Lambda(t),\mathbf{Q}(t)),    
\end{equation} 
where $\theta$ is the diagonal matrix $\diag(\theta_1, \ldots, \theta_d)$ and $\Lambda(t)=\Sigma(t)^{1/N}$. 
We obtain a tangent vector in $T_{\bfW}\balance$ by differentiating these curves
\begin{equation}
\label{eq:bb2}
D\mathfrak{z}(\bfW)(\theta, \bfa) = \left. \frac{d}{dt} \ww(t)\right|_{t=0}.    
\end{equation} 
We then find that the $p$-th matrix in $D\mathfrak{z}(\bfW)(\theta, \bfa)$ is
\begin{equation}
\label{eq:bb3} D \mathfrak{z}(\bfw)(\theta,\bfa)_p = a_p W_p + Q_p \theta Q^T_{p-1} - W_p a_{p-1}, \quad 1\leq p \leq N. 
\end{equation}
The above calculation is closely related to the computation of tangent vectors in equations~\eqref{eq:tangent1}--\eqref{eq:tangent3}. The difference is that the role of $GL(d)^{N-1}$ action on $\fiberX$ has been replaced by the parametrization $\mathfrak{z}$.  
We may again reduce the computation of an orthonormal basis for $T_\bfW \balance$ to the computations at the point $\bfC_\Sigma$ as follows. 

Recall the definition of $\bfc^{k,l,p}$, $1\leq p \leq N-1$ in equation~\qref{eq:ip1}. Now include the endpoints $p=0$ and $p=N$ by setting
\begin{equation}
    \label{eq:bb4}
    \bfc^{k,l,N} = (\alpha^{k,l}\Lambda, 0, \ldots,0), \quad \bfc^{k,l,0} = (0, \ldots,0, -\Lambda \alpha^{k,l})
\end{equation}
Also define the vectors 
\begin{equation}
    \label{eq:bb5}
    \bfm^{k} = (e_k e_k^T, e_k e_k^T, \ldots, e_ke_k^T), \quad 1\leq k \leq d.
\end{equation}
We now extend Lemma~\ref{le:basis2} to $T_{\bfC_\Sigma}\balance$.
\begin{lemma}
    \label{le:balance-basis-center}
$\{\bfc^{p,kl}\}_{0\leq p \leq N, 1\leq k < l \leq d}$ and $\{\bfm\}_{1\leq k \leq d}$ form a basis for  $T_{\bfC_\Sigma}\balance$.  
\end{lemma}
\begin{proof}
These vectors are the images of $D\mathfrak{z}(\bfW)(\theta, \bfa)$ for $\theta \in \R^d$ and $\bfa \in \anti_d^{N+1}$. Since the singular values of $X$ are distinct, $D\mathfrak{z}$ is an isomorphism between $\R^d \times \anti_d^{N+1}$ and $T_\bfW\balance$.
\end{proof}
In order to obtain an orthonormal basis, we must compute the inner products between these basis vectors to obtain the pullback metric $\mathfrak{z}^\sharp \iota$ and then diagonalize $\mathfrak{z}^\sharp \iota$. We separate these calculations into a sequence of lemmas.
\begin{lemma}
    \label{le:balance-basis-ip} For all indices $1 \leq j\leq d$, $1 \leq k <l\leq d$ and $0\leq p \leq N$
    \begin{equation}
    \langle \bfm^j, \bfc^{k,l,p}\rangle =0.
    \end{equation}
\end{lemma}
Next we observe that Lemma~\ref{le:basis2} continues to hold when $p$ and $q$ are allowed to take the endpoint values $0$ and $N$. We also find as in Lemma~\ref{le:basis2} that
\begin{lemma}
\label{le:balance-basis-ip-main1}
For all indices $0\leq p,q \leq N$, the inner-product  $\langle \bfc^{k,l,p}, \bfc^{m,n,q} \rangle =0$ if $|p-q|>1$. 
\end{lemma}
In addition to the non-zero inner-products in Lemma~\ref{le:basis3}, the only other non-zero inner products are given by
\begin{lemma}
\label{le:balance-basis-ip-main2}
For all $1\leq k < l \leq d$
\begin{equation}
    \label{eq:basis22} \langle \bfc^{k,l,0}, \bfc^{k,l,0} \rangle =\langle \bfc^{k,l,N}, \bfc^{k,l,N} \rangle = \frac{1}{2}(\lambda_k^2 + \lambda_l^2).
\end{equation}  
\begin{equation}
    \label{eq:basis23} \langle \bfc^{k,l,0}, \bfc^{k,l,1} \rangle = \langle \bfc^{k,l,N-1}, \bfc^{k,l,N} \rangle= -\lambda_k \lambda_l.
\end{equation}  
\end{lemma}
Lemma~\ref{le:balance-basis-ip-main1} and Lemma~\ref{le:balance-basis-ip-main2} are easy modifications of  Lemma~\ref{le:basis3} and the proof is omitted. The only substantive  change is with equations~\eqref{eq:basis10}--\eqref{eq:ip3}. Since $\bfc^{k,l,0}$ and $\bfc^{k,l,N}$ have non-zero entries in only one slot, the results differ by a factor of two.

 We obtain the tangent space $T_\bfW\balance$ at an arbitrary point $\bfW \in \balance$ using the $O_d^{N+1}$ action defined in equation~\eqref{eq:big-group1}. 
\begin{lemma}
    \label{le:tangent-center3}
    Assume $\bfW = \bfQ^N:\bfC_\Sigma$. The tangent space $T_\bfW\balance$ consists of
    \begin{equation}
        \label{eq:tangent-center5}
        \bfQ^N:\bfw = (Q_Nw_NQ_{N-1}^T,\ldots, Q_1 w_1 Q_0^T), \quad \bfw \in T_{\bfC_\Sigma}\balance.
    \end{equation}
\end{lemma}
\subsection{The pullback metric $\mathfrak{z}^\sharp \Fr$ and the Proof of Theorem~\ref{thm:orthonormal-basis}} 
\label{subsec:pullback}
The translation in Lemma~\ref{le:tangent-center3} is an isometry. Therefore, the above lemmas complete the computation of the metric $(\balance,\iota)$ at an arbitrary point $
\bfW$ on the balanced manifold $\balance$. We collect these results in the 
following analogue of Corollary~\ref{cor:metric-orbit}.
\begin{corollary}
 \label{cor:metric-balance}
    The standard basis on $\R^d\times \anti_d^{N+1}$ may be ordered such that the pullback metric $\mathfrak{z}^\sharp \Fr$ has the block diagonal structure
\begin{equation}
    \label{eq:pb3}
    \tilde{h} = \left( \begin{array}{cccc} 
    NI_d &   &  &\\
         & \tilde{h}_{a}^{1,2} &  &\\
             &  &\ddots     &  \\
         & & & \tilde{h}_a^{d-1,d}
    \end{array} \right),
\end{equation}
where $\tilde{h}_a^{k,l}$ is the $(N+1)\times (N+1)$ symmetric tridiagonal matrix
\begin{equation}
    \label{eq:pb4}
    \tilde{h}_a^{k,l} = \left( \begin{array}{cccccc} 
    \tfrac{1}{2}(\lambda_k^2 +\lambda_l^2) &  -\lambda_k \lambda_l &  & &\\
      -\lambda_k \lambda_l   & \lambda_k^2 +\lambda_l^2 & -\lambda_k \lambda_l & &\\
             &  -\lambda_k \lambda_l  & \lambda_k^2 +\lambda_l^2  & -\lambda_k \lambda_l  \\  
             &    & \ddots & \ddots  &   \\ \\  &    & -\lambda_k \lambda_l & \lambda_k^2 +\lambda_l^2 &  -\lambda_k \lambda_l \\
         & &  & -\lambda_k \lambda_l & \tfrac{1}{2}(\lambda_k^2 +\lambda_l^2)
    \end{array} \right).
\end{equation}
\end{corollary}
As in Corollary~\ref{cor:metric-orbit}, each block corresponds to the basis matrix $\alpha^{k,l} \in \anti_d$ and there are $d(d-1)/2$ blocks. However, in contrast with $h_a^{k,l}$ in Corollary~\ref{cor:metric-orbit}, each block $\tilde{h}_a^{k,l}$ is now of size $(N+1)\times (N+1)$ since we also include the variation of $Q_N$ and $Q_0$. The leading diagonal block $N I_d$ corresponds to the the tangent vectors $\mathbf{v}^k$.

In order to diagonalize $\tilde{h}$, we must diagonalize each tridiagonal matrix $\tilde{h}_a^{k,l}$. When compared with Lemma~\ref{le:basis2}, the main issue is to account for the endpoints of the network, $p=0$ and $p=N$. This may be seen as the imposition of boundary conditions on the chain (see Remark~\ref{rem:svd-distinct-basis} below).

The necessary modification of Lemma~\ref{le:basis6} is as follows. Define the matrix $P^{k,l}$ with entries
\begin{eqnarray}
    \label{eq:cheby4}
    P^{k,l}_{0q} & =& \frac{1}{\sqrt{\sigma^{kl}_0}} \lambda_k^q \lambda_l^{N-q},\quad 0 \leq q \leq N, \\ 
     \label{eq:cheby5}
    P^{k,l}_{pq} & = & \frac{1}{\sqrt{\sigma^{kl}_p}} S_{pq}, \quad 1\leq p \leq N-1, \; 0 \leq q \leq N\\
     \label{eq:cheby6}
    P^{k,l}_{Nq} & = & \frac{1}{\sqrt{\sigma^{kl}_N}} \lambda_k^{N-q} \lambda_l^{q}, \quad 0 \leq q \leq N.
\end{eqnarray}
Here $S\in O_{N-1}$ and $\{\sigma^{k,l}_p\}_{1\leq p \leq N-1}$ were defined in equations~\eqref{eq:cheby1}--~\eqref{eq:cheby2}, and 
\begin{equation}
\label{eq:cheby7}
 \sigma^{k,l}_0 =  \sigma^{k,l}_N = \frac{1}{2}(\lambda_l^2 -\lambda_k^2)(\lambda_l^{2N}-\lambda_k^{2N}). 
\end{equation}
\begin{lemma}
\label{le:big-jacobi}
The matrices $\tilde{h}^{k,l}_a$ and $P^{k,l}$ satisfy the identity 
\begin{equation}
    \label{eq:basis41} P^{k,l}\tilde{h}_a^{kl} (P^{k,l})^T= I_{N+1}.
\end{equation}  
\end{lemma}
We first show how Theorem~\ref{thm:orthonormal-basis} follows from Lemma~\ref{le:big-jacobi}.
\begin{proof}[Proof of Theorem~\ref{thm:orthonormal-basis}]
By Lemma~\ref{le:balance-basis-center} and Corollary~\ref{cor:metric-balance} it is enough to fix $(kl)$ and diagonalize the associated block of $\mathfrak{z}^\sharp \iota$. Lemma~\ref{le:big-jacobi} provides the desired orthonormal basis for this block. Indeed, 
\begin{equation}
    \label{eq:basis42}
    \bfu^{k,l,p} = \sum_{r=0}^N P^{k,l}_{pr}\bfc^{k,l,r}, \quad 0 \leq p \leq N.
\end{equation}
Then it is immediate that for $0\leq p,q\leq N$
\begin{equation}
    \label{eq:basis43}
    \langle \bfu^{k,l,p}, \bfu^{k,l,q} \rangle = \sum_{r,s=0}^N P^{k,l}_{pr}\langle \bfc^{k,l,r},\bfc^{k,l,s}\rangle P^{k,l}_{qs} = (P^{k,l} \tilde{h}_a^{k,l} (P^{k,l})^T)_{pq}=\delta_{pq}.
\end{equation}
\end{proof}
\begin{proof}[Proof of Lemma~\ref{le:big-jacobi}]
Let $T$ be the $(N+1)\times (N+1)$ matrix obtained by extending equation~\eqref{eq:cheby1} to $p=0$ and $p=N$. The top and bottom rows and leftmost and rightmost columns of $T$ vanish. Thus, the columns $\{t_1, \ldots, t_{N-1}\}$ are eigenvectors of $\tilde{h}_a^{k,l}$ with the eigenvalues $\{\sigma_p^{k,l}\}_{p=1}^{N-1}$ defined in equation~\eqref{eq:cheby2}. 

Thus, we only need to determine two additional vectors $a$ and $b$ that do not lie in the span of $\{t_1, \ldots, t_{N-1}\}$. We observe that the geometric sequences
\begin{equation}
    \label{eq:basis51}
    a_{p} = \lambda_k^p  \lambda_l^{N-p}, \quad b_p = \lambda_k^{N-p}\lambda_l^p, \quad 0 \leq p \leq N, 
\end{equation}
have the property that
\begin{equation}
    \label{eq:basis52}
    \tilde{h}_a^{k,l} \left(\begin{array}{c}a_0 \\ a_1 \\ \vdots \\ a_{N-1} \\ a_N \end{array} \right) = \frac{1}{2}(\lambda_l^2 -\lambda_k^2) 
    \left(\begin{array}{c}a_0 \\ 0 \\ \vdots \\ 0 \\ -a_{N} \end{array} \right),
\end{equation}
and 
\begin{equation}
    \label{eq:basis53}
    \tilde{h}_a^{k,l} \left(\begin{array}{c}b_0 \\ b_1 \\ \vdots \\ b_{N-1} \\ b_N \end{array} \right) = \frac{1}{2}(\lambda_l^2 -\lambda_k^2) 
    \left(\begin{array}{c}-b_0 \\ 0 \\ \vdots \\ 0 \\ b_{N} \end{array} \right).
\end{equation}
We immediately obtain the orthogonality property
\begin{equation}
    \label{eq:basis54}
    t_p^T \tilde{h}_a^{k,l} a =  t_p^T \tilde{h}_a^{k,l} b =0, \quad 1\leq p \leq N-1.
\end{equation}
We also note that $a$ and $b$ are themselves orthogonal, since
\begin{equation}
    \label{eq:basis55}
    b^T \tilde{h}_a^{k,l} a =  \frac{1}{2}(\lambda_l^2 -\lambda_k^2)(a_0b_0-a_Nb_N)=0.
\end{equation}
We normalize the column vectors $a$ and $b$ to obtain the matrix $P^{k,l}$.
\end{proof}

\begin{remark}
    \label{rem:svd-distinct-basis} The following contrast with equation~\eqref{eq:ip6} may help the reader understand the nature of the eigenvectors corresponding to $p=0$ and $p=N$.    When $\lambda_k=\lambda_l=1$, $\tilde{h}_a^{k,l}$ reduces to the matrix \begin{equation}
    \label{eq:pb5}
    \tilde{h}_a^{k,l} = \left( \begin{array}{cccccc} 
    1 &  -1  &  & &\\
      -1  & 2 & -1 & &\\
             &  -1  & 2  & -1  \\  
             &    & \ddots & \ddots  &   \\ \\  &    & -1 & 2 &  -1 \\
         & &  & -1 & 1
    \end{array} \right).
\end{equation}
While this matrix appears in the theory of Chebyshev polynomials, it is singular. Indeed, it is easy to check that the vector $(1,\ldots,1)$ lies in the nullspace of $\tilde{h}_a^{k,l} $. This matrix corresponds to Neumann boundary conditions for the boundary value problem with second-order difference operator defined by $h^{k,l}$. The reader is invited to compute the eigenvalues and eigenvectors as a limit of Lemma~\ref{le:big-jacobi}.
\end{remark}

\subsection{Riemannian submersion}
\label{subsec:submersion}
\begin{proof}[Proof of Theorem~\ref{thm:submersion}]
Recall the definition of the metric $g^N$ in equations~\eqref{eq:metric2}--~\eqref{eq:metric3} in Section~\ref{subsec:prior}. In order to show that $g^N$ is obtained through Riemannian submersion we must show that $\phi_*: (\mathrm{Ker}\,\phi_*)^{\perp}\to \mathfrak{M}_d$ is an isometry, where $\phi_*:T\mathcal M\to T\mathfrak{M}$ is the differential of $\phi$.

To this end, we observe that Theorem~\ref{thm:orthonormal-basis} allows us to compute $\mathrm{Ker}\phi_*$ and $\mathrm{Ker}\phi_*^\perp$ explicitly. First we note that
\begin{equation}
    \label{eq:submersion1} (\mathrm{Ker}\,\phi_*) = \mathrm{span}\{\bfu^{k,l,p}, 1\leq p\leq N-1, \; 1\leq k< l \leq d \}.
\end{equation}
Indeed, these vectors generate the $O_d^{N-1}$ group action that preserve $\orbitx$. Thus, $X=\phi(\bfW(t))$ is constant for any one-parameter curve obtained by integration along these vector fields. Since $\mathrm{dim}(\balance)=d^2+(N-1) d(d-1)/2$ and $\mathrm{dim}(\mathfrak{M}_d)=d^2$ the remaining vectors $\{\bfv^k\}_{k=1}^d$ and $\{\bfu^{k,l,p}, p=0,\,N, \; 1\leq k< l \leq d \}$ form an orthonormal basis for $(\mathrm{Ker}\,\phi_*)^{\perp}$. 

We must now show that $\phi_*: (\mathrm{Ker}\,\phi_*)^{\perp}\to \mathfrak{M}_d$ is an isometry. This calculation is simplified through the calculus of differential forms. 
By definition $\phi(\bfW)=X$, so that $\phi_*(\bfu^{k,l,p})=\bfu^{k,l,p}(X)$. Similarly, $\phi_*(\bfv)=\bfv(X)$. Let us also recall the notation of the SVD of $X$ in equation~\eqref{eq:diagon1}.

We use the definitions of $\bfv^k$ and $\bfu^{k,l,p}$ for $p=0$ and $N$ in equations~\eqref{eq:vb6} and~\eqref{eq:vb7} to find that for $1\leq k < l \leq d$
\begin{align}
\label{eq:submersion2}
        \bfv^k(X)=&\sqrt{N}\lambda_k^{N-1} q_{N,k}q_{0,k}^T\\
        \label{eq:submersion3}
        \bfu^{k,l,0}(X)=&\sqrt{\frac{\lambda_k^{2N}-\lambda_l^{2N}}{\lambda_k^2-\lambda_l^2}}q_{N,l}q_{0,k}^T\\
        \label{eq:submersion4}
        \bfu^{k,l,N}(X)=&\sqrt{\frac{\lambda_k^{2N}-\lambda_l^{2N}}{\lambda_k^2-\lambda_l^2}}q_{N,k}q_{0,l}^T.
    \end{align}
    Since $\Lambda^N=\Sigma$, this is exactly the decomposition of $g^N$ stated in equation~\eqref{eq:p-metric1} following Lemma~\eqref{le:diagon}. 
\end{proof}


\section{Discussion}
\label{sec:discussion}
\subsection{Overview}
Our purpose in this work has been to shed new light on the following question: how does depth affect the training dynamics in deep learning? 

We have studied this problem in the simplified setting of the DLN using methods from control theory, dynamical systems theory, and random matrix theory. In this section, we recall the notion of implicit bias in deep learning. We then explain (first as a recipe and then in the concepts of equilibrium thermodynamics) the relation between the entropy formula, gradient dynamics and implicit bias. 



\subsection{Implicit bias}
Let us first recall the notion of implicit bias within the framework of statistical learning theory. 

The task in statistical learning theory is to construct a function $f: \R^{d_1}\to \R^{d_2}$ in a parametrized class of functions from given training data $\{(x_j,y_j)\}_{j=1}^n \subset   \R^{d_1}\times \R^{d_2}$.  Once $f$ has been constructed, its performance is evaluated on a distinct set of test data, and we say that $f$ generalizes well when it approximates the test data with small error. 

In deep learning, the function $f$ is constructed using neural networks whose parameters are obtained by using gradient descent to minimize the empirical loss function constructed from the data. The fundamental mystery in deep learning is why deep neural networks generalize so well, without overfitting, even in the absence of an explicit regularizer. A striking dynamic demonstration of the distinction between the predictions of (traditional) statistical learning theory and deep learning is provided by the double-descent curve of Belkin {\em et al}~\cite{Belkin-dd,Belkin-dd2}. This ability of deep neural networks to generalize without overfitting is known as implicit bias or implicit regularization.  We refer the reader to a recent review by Vardi that summarizes several approaches to implicit bias~\cite{Vardi}. 




The main conjecture that underlies our work is that implicit bias in deep learning is a dynamic phenomena that originates in the geometry of gradient fields for overparametrized functions. To this end, we have focused on a rigorous analysis of the geometric consequences of overparametrization in the DLN using Lie groups. We have identified robust properties of the gradient flow that do {\em not\/} rely on the form of the loss function or on simplifying assumptions on the dimension and depth. The main advantages of our approach are that it provides a deeper understanding of the fundamental concept of balancedness, it reveals the intrinsic Riemannian geometry of the learning process, it allows the inclusion of noise as a selection principle, and that it provides exact formulas (such as the entropy formula) without any restrictions on $d$ and $N$. We note that there have been other recent attempts to derive effective theories for deep learning using geometric methods (see in particular~\cite{t-chen}).

The mathematical foundation of our approach is dynamical systems theory, not statistical learning theory. But there are, of course, several places where the methods overlap. In addition to the prior work~\cite{ACH,Bah} whose importance we have discussed in Section~\ref{sec:results}, there have been several attempts to understand implicit bias in linear networks. These include a characterization under simplifying assumptions (with margin-based generalization bounds for logistic regression when $d=1$~\cite{Telgarsky} and with the nuclear norm when $N=2$~\cite{Gunasekar2018,Gunasekar2017}). An interesting recent modification of linear networks are the spectral neural networks introduced in~\cite{ghosh}. There have also been several studies of the convergence to critical points  in the DLN for different choices of learning tasks~\cite{Brechet-Montufar,Ge1,Rauhut2} or simplifying assumptions on dimension (e.g. $N=3$ and $d=\infty$ in~\cite{Figalli}). There has also been rigorous geometric analysis of the expressivity properties of linear convolutional networks~\cite{Montufar1,Montufar2}. It is of interest to combine this geometric analysis with gradient dynamics.

\subsection{Entropic regularization}
\label{subsec:regulatization}
The entropy formula may be used to regularize the gradient flow by modifying the gradient dynamics to include small noise in the gauge group $O_d^{N-1}$. While the rigorous theory is developed in~\cite{MY-RLE}, we describe the main insights here. 

Recall the Boltzmann entropy computed in Theorem~\ref{thm:entropy}. 
Given a loss function $E(X)$ and an inverse temperature $\beta \in (0,\infty)$ we may now define the {\em free energy\/}
\begin{equation}
    \label{eq:free-energy}
    F_\beta(X) = E(X) - \frac{1}{\beta} S(X).
\end{equation}
The gradient flow of $E$ in Theorem~\ref{thm:BRTW2} may be augmented to the gradient flow of the free energy
\begin{equation}
\label{eq:free-energy-descent}
\dot{X} = -\mathrm{grad}_{g^N} F_\beta (X).
\end{equation}
We show in~\cite{MY-RLE} (see also~\cite[\S 12]{GM-dln} for a simplified exposition) that the gradient dynamics of $X$ correspond to {\em stochastic\/} dynamics for $\ww$ on the balanced manifold $\mathcal{M}$ described by a Riemannian Langevin equation (RLE). The stochastic forcing of this RLE is Brownian motion on $\orbitx$ and an explicit description of this Brownian motion is provided by the orthonormal basis computed in Theorem~\ref{thm:orthonormal-basis}.

It is shown in~\cite{Chen2} that the entropy is concave (in the Euclidean geometry on $\Md)$. Thus, the inclusion of the entropy in the free energy regularizes the loss function $E(X)$ allowing us to characterize the minimizers of $F_\beta$ for certain loss functions $E$. While the entropy formula was discovered through an analogy with random matrix theory (RMT), a surprising feature is that the minimizers of the free energy for the DLN do {\em not\/} show singular value repulsion.

The small noise limit is $\beta\to \infty$. We conjecture that the inclusion of the entropy provides a {\em selection principle\/} when $E(X)$ has a large family of minimizers. Such degenerate loss functions arise in learning tasks such as matrix sensing and matrix completion. In this setting, we conjecture that the observed minima of gradient descent for $E$ are obtained as the $\beta \to \infty$ limit of the minimizing sets
\begin{equation}
    \label{eq:free-energy3}
    \mathcal{S}_\beta= \mathrm{argmin}_{X\in \Mdd} F_\beta(X).
\end{equation}
Partial results in this direction are presented in~\cite{Chen2}.

\subsection{Equilibrium thermodynamics in the DLN}
\label{subsec:thermo}
The DLN is a gradient flow that models a learning process. On the other hand, the Boltzmann entropy, and the Boltzmann formula $S= k \log \#$ where $\#$ is the number of microstates, are concepts from statistical physics. {\em A priori\/} the DLN is not  a physical system; thus the use of terms such as microstate, entropy and thermodynamics within the context of learning requires  a careful explanation. 

We relate the learning process to thermodynamics using information theory. To this end, let us recall some fundamental principles from information theory. First, we consider the Shannon entropy for a discrete random variable. Given a random variable $X$ taking values in a finite alphabet 
$\mathcal{A}$ of size $m$ with probability $p_1,\ldots, p_m$, the Shannon entropy of $X$ is 
$S(X)= \sum_{i=1}^m p_i \log_2 p_i$. This definition extends naturally to the entropy rate for stationary sequences $\{X_n\}_{n=-\infty}^\infty$. The asymptotic equipartition property then states, roughly speaking, that the sequence $\{X_n\}_{n=p}^{p+N}$ in a window of large length $N$, is one of roughly $2^{NS}$ statistically indistinguishable sequences of $2^{Nm}$ possible sequences. Finally, both these notions may be extended to a stationary real-valued process 
$\{X(t)\}_{t\in \R}$ where $X(t)$ takes values in a manifold. Thus, the Shannon entropy is a well-defined concept for stationary processes taking values in a manifold. A minor difference in convention is the base of the logarithm; for discrete random variable $\log_2$ is common, for continuous time we use the natural logarithm $\log$.

In particular, we may always associate a Shannon entropy to Brownian motion on a compact Riemannian manifold $(\mathcal{M},g)$. In this setting, the entropy of the random variable $X(t)$ for any fixed $t$ is $S(X) = \log \mathrm{vol}_g(\mathcal{M})$ as is the entropy rate of the sequence $\{X(t)\}_{t\in\R}$. Thus, the Boltzmann entropy and the Shannon entropy are equivalent for Brownian motion on a compact Riemannian manifold $(\mathcal{M},g)$. The use of Brownian motion also provides an intuitive notion of microscopic fluctuations so that we continue to have a `physical cartoon'.

Thus, the important mathematical structure that allows us to treat the physical and information theoretic entropy in the same manner is Brownian motion on a manifold. In the setting of the DLN, Brownian motion in $(\orbitx,\iota)$ provides a precise description of fluctuations in the gauge group $O_d^{N-1}$. This ensemble of paths is our thermodynamic system and it is conveniently visualized through Riemannian submersion. As $X$ evolves downstairs according to the gradient flow~\eqref{eq:free-energy-descent}, the ensemble of paths upstairs evolves in a quasistatic manner. This physical caricature, along with its rigorous mathematical formulation, is described in detail in the forthcoming paper~\cite{MY-RLE}.


\section{Acknowledgements}
The authors thank Yotam Alexander, Sanjeev Arora, Alan Chen, Nadav Cohen, Luis Contreras Moreno, Tejas Kotwal, Kathryn Lindsey, Marco Nahas and the anonymous referee for remarks that have improved this work.

\bibliographystyle{siam}
\bibliography{dln-entropy}

\end{document}